%% file: onlineRL-arxiv.tex
\documentclass{article}
\usepackage{colt11e}
\usepackage{hyperref}
\usepackage{enumerate}
\usepackage[round,comma]{natbib}
\usepackage{amsthm}

\input{Commands}

\title{Online Learning in Markov Decision Processes with Adversarially Chosen Transition Probability Distributions}

\author{Yasin Abbasi-Yadkori \\
\small QUT \\
\small \texttt{yasin.abbasiyadkori@qut.edu.au}
\And
Peter L. Bartlett \\
\small UC Berkeley and QUT \\
\small \texttt{bartlett@eecs.berkeley.edu}\\
\And
Csaba Szepesv\'ari\\
\small University of Alberta \\
\small \texttt{szepesva@cs.ualberta.ca}
}

\begin{document}

\maketitle

\begin{abstract}
We study the problem of learning Markov decision processes with finite state and action spaces when the transition probability distributions and loss functions are chosen adversarially and are allowed to change with time. We introduce an algorithm whose regret with respect to any policy in a comparison class grows as the square root of the number of rounds of the game, provided the transition probabilities satisfy a uniform mixing condition. Our approach is efficient as long as the comparison class is polynomial and we can compute expectations over sample paths for each policy.
Designing an efficient algorithm with small regret for the general case remains an open problem.
\end{abstract}


\section{Notation}
Let $\cX$ be a finite state space and $\cA$ be a finite action space. Let $\Delta_S$ be the space of probability distributions over set $S$. Define a policy $\pi$ as a mapping from the state space to $\Delta_{\cA}$, $\pi:\cX\ra\Delta_{\cA}$. We use $\pi(a|x)$ to denote the probability of choosing action $a$ in state $x$ under policy $\pi$. A random action under policy $\pi$ is denoted by $\pi(x)$. A transition probability kernel (or transition model) $m$ is a mapping from the direct product of the state and action spaces to $\Delta_{\cX}$: $m:\cX \times \cA \ra \Delta_{\cX}$. Let $P(\pi,m)$ be the transition probability matrix of policy $\pi$ under transition model $m$. A loss function is a bounded real-valued function over state and action spaces, $\ell:\cX\times \cA \ra \Real$. For a vector $v$, define $\norm{v}_1  = \sum_i \abs{v_i}$. For a real-valued function $f$ defined over $\cX\times \cA$, define $\norm{f}_{\infty,1} = \max_{x\in \cX} \sum_{a\in \cA} \abs{f(x,a)}$. The inner product between two vectors $v$ and $w$ is denoted by $\ip{v}{w}$.


\section{Introduction}

Consider the following game between a learner and an adversary: at round $t$, the learner chooses a policy $\pi_t$ from a policy class $\Pi$. In response, the adversary chooses a transition model $m_t$ from a set of models $M$ and a loss function $\ell_t$. The learner takes action $a_t  \sim \pi_t(.|x_t)$, moves to state $x_{t+1} \sim m_t(.|x_t,a_t)$ and suffers loss $\ell_t(x_t, a_t)$. To simplify the discussion, we assume that the adversary is oblivious, i.e. its choices do not depend on the previous choices of the learner. We assume that $\ell_t \in [0,1]$. In this paper, we study the full-information version of the game, where the learner observes the transition model $m_t$ and the loss function $\ell_t$ at the end of round $t$. The game is shown in Figure~\ref{alg:online-mdp-game}. The objective of the learner is to suffer low loss over a period of $T$ rounds, while the performance of the learner is measured using its regret with respect to the total loss he would have achieved had he followed the stationary policy in the comparison class $\Pi$ minimizing the total loss.

\citet{Even-Dar-Kakade-Mansour-2004} prove a hardness result for MDP problems with adversarially chosen transition models. Their proof, however, seems to have gaps as it assumes that the learner chooses a deterministic policy before observing the state at each round. Note that an online learning algorithm only needs to choose an action at the current state and does not need to construct a complete deterministic policy at each round. Their hardness result applies to deterministic transition models, while we make a mixing assumption in our analysis. Thus, it is still an open problem whether it is possible to obtain a computationally efficient algorithm with a sublinear regret. 

\citet{Yu-Mannor-2009,Yu-Mannor-2009-b} study the same setting, but obtain only a regret bound that scales with the amount of variation in the transition models. This regret bound can grow linearly with time. 

\citet{Even-Dar-Kakade-Mansour-2009} prove regret bounds for MDP problems with a fixed and known transition model and adversarially chosen loss functions. In this paper, we prove regret bounds for MDP problems with adversarially chosen transition models and loss functions. We are not aware of any earlier regret bound for this setting. Our approach is efficient as long as the comparison class is polynomial and we can compute expectations over sample paths for each policy.



MDPs with changing transition kernels are good models for a wide range of problems, including dialogue systems, clinical trials, portfolio optimization, two player games such as poker, etc. 




\section{Online MDP Problems}
\label{sec:online-mdp}

\begin{figure}
\begin{center}
\framebox{\parbox{12cm}{
\begin{algorithmic}
\STATE Initial state: $x_0$
\FOR{$t:=1,2,\dots$}
\STATE Learner chooses policy $\pi_t$
\STATE Adversary chooses model $m_t$ and loss function $\ell_t$
\STATE Learner takes action $a_t \sim \pi_t(.|x_t)$
\STATE Learner suffers loss $\ell_t(x_t, a_t)$
\STATE Update state $x_{t+1} \sim m_t(.|x_t,a_t)$
\STATE Learner observes $m_t$ and $\ell_t$
\ENDFOR
\end{algorithmic}
}}
\end{center}
\caption{Online Markov Decision Processes}
\label{alg:online-mdp-game}
\end{figure}

Let $A$ be an online learning algorithm that generates a policy $\pi_t$ at round $t$. Let $x_t^A$ be the state at round $t$ if we have followed the policies generated by algorithm $A$. Similarly, $x_t^\pi$ denotes the state if we have chosen the same policy $\pi$ up to time $t$. 
Let $\ell(x,\pi) = \ell(x,\pi(x))$. The regret of algorithm $A$ up to round $T$ with respect to any policy $\pi\in\Pi$ is defined by
\[
R_T(A, \pi) = \sum_{t=1}^T \ell_t(x_t^A, a_t) - \sum_{t=1}^T \ell_t(x_t^\pi, \pi) \, ,
\]
where $a_t = \pi_t(x_t^A)$. Note that  the regret with respect to $\pi$ is defined in terms of the sequence of states $x_t^\pi$ that would have been visited under policy $\pi$. Our objective is to design an algorithm that achieves low regret with respect to any policy $\pi$. 

In the absence of state variables, the problem reduces to a \textit{full information online learning problem} \citep{Cesa-Bianchi-Lugosi-2006}. The difficulty with MDP problems is that, unlike the full information online learning problems, the choice of policy at each round changes the future states and losses. The main idea behind the design and the analysis of our algorithm is the following regret decomposition:
\beq
\label{eq:decomposition}
R_T(A,\pi) = \sum_{t=1}^T \ell_t(x_t^A, a_t) - \sum_{t=1}^T \ell_t(x_t^{\pi_t}, \pi_t) + \sum_{t=1}^T \ell_t(x_t^{\pi_t}, \pi_t) - \sum_{t=1}^T \ell_t(x_t^\pi, \pi) \; .
\eeq
Let
\begin{align*}
B_T(A) &= \sum_{t=1}^T \ell_t(x_t^A, a_t) - \sum_{t=1}^T \ell_t(x_t^{\pi_t}, \pi_t) \,, \\
C_T(A,\pi) &= \sum_{t=1}^T \ell_t(x_t^{\pi_t}, \pi_t) - \sum_{t=1}^T \ell_t(x_t^\pi, \pi)\; .
\end{align*}
Notice that the choice of policies has no influence over future losses in $C_T(A,\pi)$. Thus, $C_T(A,\pi)$ can be bounded by a specific reduction to full information online learning algorithms (to be specified later). 
Also, notice that the competitor policy $\pi$ does not appear in $B_T(A)$. In fact, $B_T(A)$ depends only on the algorithm $A$. We will show that if algorithm $A$ and the class of models satisfy the following two ``smoothness'' assumptions, then $B_T(A)$ can be bounded by a sublinear term.
\begin{ass}[Rarely Changing Policies]
\label{ass:slow-changing-policy}
Let $\alpha_t$ be the probability that algorithm $A$ changes its policy at round $t$. There exists a constant $D$ such that for any $1\le t \le T$, any sequence of models $m_1,\dots, m_t$ and loss functions $\ell_1,\dots,\ell_t$, $\alpha_t \le D/\sqrt{t}$.
\end{ass}
\begin{ass}[Uniform Mixing]
\label{ass:uniform-mixing}
There exists a constant $\tau>0$ such that for all distributions $d$ and $d'$ over the state space, any deterministic policy $\pi$, and any model $m\in M$,
\[
\norm{d P(\pi,m) - d' P(\pi,m)}_1 \le e^{-1/\tau} \norm{d - d'}_1 \;.
\]
\end{ass}
As discussed by \citet{Neu-Gyorgy-Szepesvari-Antos-2010}, if Assumption~\ref{ass:uniform-mixing} holds for deterministic policies, then it holds for all policies. 

\subsection{Full Information Algorithms}

\begin{figure}
\begin{center}
\framebox{\parbox{12cm}{
\begin{algorithmic}
\STATE $N$: number of experts, $T$: number of rounds.
\STATE Initialize $w_{i,0} = 1$ for each expert $i$.
\STATE $W_0 = N$.
\FOR{$t:=1,2,\dots$}
\STATE For any $i$, $p_{i,t} = w_{i,t-1}/W_{t-1}$.
\STATE Draw $I_t$ such that for any $i$, $\Prob{I_t = i} = p_{i,t}$.
\STATE Choose the action suggested by expert $I_t$.
\STATE The adversary chooses loss function $c_t$.
\STATE The learner suffers loss $c_t(I_t)$.
\STATE For expert $i$, $w_{i,t} = w_{i,t-1} e^{-\eta c_t(i)}$.
\STATE $W_t = \sum_{i=1}^N w_{i,t}$.
\ENDFOR
\end{algorithmic}
}}
\end{center}
\caption{The \textsc{EWA} Algorithm}
\label{fig:ewa}
\end{figure}

We would like to have a full information online learning algorithm that rarely changes its policy. The first candidate that we consider is the well-known Exponentially Weighted Average (\textsc{EWA}) algorithm~\citep{Vovk-1990,Littlestone-Warmuth-1994} shown in Figure~\ref{fig:ewa}. In our MDP problem, the \textsc{EWA} algorithm chooses a policy $\pi\in\Pi$ according to distribution
\beq
\label{eq:policy}
q_t(\pi) \propto \exp\left( -\lambda \sum_{s=1}^{t-1} \EE{\ell_s(x_s^{\pi}, \pi)} \right),\quad \lambda > 0\,,
\eeq
The policies that this \textsc{EWA} algorithm generates most likely are different in consecutive rounds and thus, the \textsc{EWA} algorithm might change its policy frequently. However, a variant of \textsc{EWA}, called Shrinking Dartboard (SD)~\citep{Geulen-Vocking-Winkler-2010} and shown in Figure~\ref{alg:SD}, satisfies Assumption~\ref{ass:slow-changing-policy}. Our algorithm, called \textsc{SD-MDP}, is based on the \textsc{SD} algorithm and is shown in Figure~\ref{alg:SD-MDP}. Notice that the algorithm needs to know the number of rounds, $T$, in advance.

\begin{figure}
\begin{center}
\framebox{\parbox{12cm}{
\begin{algorithmic}
\STATE $N$: number of experts, $T$: number of rounds.
\STATE $\eta = \min\{\sqrt{\log N / T}, 1/2\}$.
\STATE Initialize $w_{i,0} = 1$ for each expert $i$.
\STATE $W_0 = N$.
\FOR{$t:=1,2,\dots$}
\STATE For any $i$, $p_{i,t} = w_{i,t-1}/W_{t-1}$.
\STATE With probability $\beta_t = w_{I_{t-1},t-1}/w_{I_{t-1},t-2}$ choose the previously selected expert, $I_t = I_{t-1}$ and with probability $1-\beta_t$, choose $I_t$ based on the distribution $q_t = (p_{1,t},\dots, p_{N,t})$.
\STATE Learner takes the action suggested by expert $I_t$.
\STATE The adversary chooses loss function $c_t$.
\STATE The learner suffers loss $c_t(I_t)$.
\STATE For all experts $i$, $w_{i,t} =  w_{i,t-1} (1-\eta)^{c_{t}(i)}$.
\STATE $W_t = \sum_{i=1}^N w_{i,t}$.
\ENDFOR
\end{algorithmic}
}}
\end{center}
\caption{The Shrinking Dartboard Algorithm}
\label{alg:SD}
\end{figure}

Consider a basic full information problem with $N$ experts. Let $R_T(\textsc{SD}, i)$ be the regret of the \textsc{SD} algorithm with respect to expert $i$ up to time $T$. We have the following results for the \textsc{SD} algorithm.
\begin{thm}
\label{thm:regret-SD}
For any expert $i\in \{1,\dots,N\}$,
\[
R_T(\textsc{SD}, i) \le 4 \sqrt{T \log N} + \log N \,,
\]
and also for any $1\le t\le T$,
\[
\Prob{\hbox{Switch at time $t$}} \le \sqrt{\frac{\log N}{T}}\;.
\]
\end{thm}
\begin{proof}
The proof of the regret bound can be found in \citep[Theorem 3]{Geulen-Vocking-Winkler-2010}. The proof of the bound on the probability of switch is similar to the proof of Lemma 2 in \citep{Geulen-Vocking-Winkler-2010} and is as follows: As shown in \citep[Lemma 2]{Geulen-Vocking-Winkler-2010}, the probability of switch at time $t$ is 
\[
\alpha_t = \frac{W_{t-1} - W_t}{W_{t-1}}\;.
\]
Thus, $W_t = (1-\alpha_t) W_{t-1}$. Because the loss function is bounded in $[0,1]$, we have that 
\[
W_t = \sum_{i=1}^N w_{i,t} = \sum_{i=1}^N w_{i,t-1} (1-\eta)^{c_{t}(i)} \ge \sum_{i=1}^N w_{i,t-1} (1-\eta) = (1-\eta) W_{t-1}\;.
\]
Thus, $1-\alpha_t \ge 1-\eta$, and thus,
\[
\alpha_t \le \eta \le \sqrt{\frac{\log N}{T}}\;.
\]
\end{proof}

\begin{figure}
\begin{center}
\framebox{\parbox{12cm}{
\begin{algorithmic}
\STATE $T$: number of rounds.
\STATE $\eta = \min\{\sqrt{\log \abs{\Pi} / T}, 1/2\}$.
\STATE For all policies $\pi\in\{1,\dots, \abs{\Pi}\}$, $w_{\pi,0} = 1$.
\FOR{$t:=1,2,\dots$}
\STATE For any $\pi$, $p_{\pi,t} = w_{\pi,t-1}/W_{t-1}$.
\STATE With probability $\beta_t = w_{\pi_{t-1},t-1}/w_{\pi_{t-1},t-2}$ choose the previous policy, $\pi_t = \pi_{t-1}$, while with probability $1-\beta_t$, choose $\pi_t$ based on the distribution $q_t = (p_{1,t},\dots, p_{\abs{\Pi},t})$.
\STATE Learner takes the action $a_t \sim \pi_t(.|x_t)$
\STATE Adversary chooses transition model $m_t$ and loss function $\ell_t$.
\STATE Learner suffers loss $\ell_t(x_t, a_t)$.
\STATE Learner observes $m_t$ and $\ell_t$.
\STATE Update state: $x_{t+1} \sim m_t(.|x_t,a_t)$.
\STATE For all policies $\pi$, $w_{\pi,t} =  w_{\pi,t-1} (1-\eta)^{\EE{\ell_{t}(x_{t}^\pi, \pi)}}$.
\STATE $W_t = \sum_{\pi\in\Pi} w_{\pi,t}$.
\ENDFOR
\end{algorithmic}
}}
\end{center}
\caption{SD-MDP: The Shrinking Dartboard Algorithm for Markov Decision Processes}
\label{alg:SD-MDP}
\end{figure}

\subsection{Analysis of the \textsc{SD-MDP} Algorithm}

The main result of this section is the following regret bound for the \textsc{SD-MDP} algorithm.
\begin{thm}
\label{thm:main}
Let the loss functions selected by the adversary be bounded in $[0,1]$, and the transition models selected by the adversary satisfy Assumption~\ref{ass:uniform-mixing}. Then, for any policy $\pi\in\Pi$,
\[
\EE{R_T(\textsc{SD-MDP},\pi)} \le (4 + 2 \tau^2) \sqrt{T \log |\Pi|} + \log |\Pi|  \; .
\]
\end{thm}
In the rest of this section, we write $A$ to denote the \textsc{SD-MDP} algorithm. For the proof we use the regret decomposition \eqref{eq:decomposition}:
\[
R_T(A,\pi) = B_T(A) + C_T(A,\pi) \; .
\]

\subsubsection{Bounding $\EE{C_T(A,\pi)}$}

\begin{lem}
\label{lem:second-part}
For any policy $\pi\in\Pi$,
\[
\EE{C_T(A,\pi)}=\EE{\sum_{t=1}^T \ell_t(x_t^{\pi_t}, \pi_t) - \sum_{t=1}^T  \ell_t(x_t^\pi, \pi)} \le 4 \sqrt{T \log |\Pi|} + \log |\Pi|\; .
\]
\end{lem}
\begin{proof}
Consider the following imaginary game between a learner and an adversary: we have a set of experts (policies) $\Pi = \{\pi^1,\dots, \pi^{|\Pi|}\}$. At round $t$, the adversary chooses a loss vector $c_t\in [0,1]^{\Pi}$, whose $i$th element determines the loss of expert $\pi^i$ at this round. The learner chooses a distribution over experts $q_t$ (defined by the \textsc{SD} algorithm), from which it draws an expert $\pi_t$. Next, the learner observes the loss function $c_t$. From the regret bound for the \textsc{SD} algorithm (Theorem~\ref{thm:regret-SD}), it is guaranteed that for any expert $\pi$,
\[
\sum_{t=1}^T \ip{c_t}{q_t} - \sum_{t=1}^T c_t(\pi) \le 4 \sqrt{T \log |\Pi|} + \log |\Pi|\; .
\]
Next, we determine how the adversary chooses the loss vector. At time $t$, the adversary chooses a loss function $\ell_t$ and sets $c_t(\pi^i) = \EE{\ell_t(x_t^{\pi^i}, \pi^i)}$. Noting that $\ip{c_t}{q_t} = \EE{\ell_t(x_t^{\pi_t}, \pi_t)}$ and $c_t(\pi)=\EE{\ell_t(x_t^{\pi}, \pi)}$ finishes the proof.
\end{proof}

\subsubsection{Bounding $\EE{B_T(A)}$}

First, we prove the following two lemmas.
\begin{lem}
For any state distribution $d$, any transition model $m$, and any policies $\pi$ and $\pi'$,
\[
\norm{d P(\pi,m) - d P(\pi',m)}_1\le \norm{\pi - \pi^{\prime}}_{\infty,1} \; .
\]
\end{lem}
\begin{proof}
Proof is easy and can be found in \citep{Even-Dar-Kakade-Mansour-2009}, Lemma~5.1.
\end{proof}

\begin{lem}
\label{lem:policy-switch}
Let $\alpha_t$ be the probability of a policy switch at time $t$. Then, $\alpha_t \le \sqrt{\log |\Pi|/T} $.
\end{lem}
\begin{proof}
Proof is identical to the proof of Theorem~\ref{thm:regret-SD}.
\end{proof}

\begin{lem}
\label{lem:first-part}
We have that
\[
\EE{B_T(A)} = \EE{\sum_{t=1}^T \ell_t(x_t^A, a_t) - \sum_{t=1}^T \ell_t(x_t^{\pi_t}, \pi_t)} \le 2 \tau^2 \sqrt{\log |\Pi|T} \; .
\]
\end{lem}
\begin{proof}
Let $\cF_t = \sigma(\pi_{1}, \dots, \pi_t)$. Notice that the choice of policies are independent of the state variables. 
We can write
\begin{align}
\label{eq:BT}
\notag
\EE{B_T(A)} &= \EE{\sum_{t=1}^T \ell_t(x_t^A, a_t) - \sum_{t=1}^T \ell_t(x_t^{\pi_t}, \pi_t)} \\
\notag
&= \EE{\sum_{t=1}^T \sum_{x\in\cX} \left(\one{x_t^A = x} - \one{x_t^{\pi_t} = x}\right) \ell_t(x,\pi_t(x))} \\
\notag
&= \EE{\sum_{t=1}^T \sum_{x\in\cX} \EE{ \left(\one{x_t^A = x} - \one{x_t^{\pi_t} = x}\right) \ell_t(x,\pi_t(x)) \, \middle|\, \cF_T }} \\
\notag
&= \EE{\sum_{t=1}^T \sum_{x\in\cX} \ell_t(x,\pi_t(x)) \EE{ \left(\one{x_t^A = x} - \one{x_t^{\pi_t} = x}\right) \, \middle|\, \cF_T }} \\
\notag
&\le \EE{\sum_{t=1}^T  \norm{\ell_t}_{\infty} \norm{\EE{ \left(\one{x_t^A = x} - \one{x_t^{\pi_t} = x}\right) \, \middle|\, \cF_T }}_1 } \\
\notag
&= \EE{\sum_{t=1}^T \norm{\ell_t }_{\infty} \norm{u_{t} - v_{t,t}}_1} \\
&\le \EE{\sum_{t=1}^T \norm{u_{t} - v_{t,t}}_1} \, ,
\end{align}
where $u_{s}=\EE{\one{x_s^A = x} \middle| \cF_T}$ is the distribution of $x_s^A$ for $s\le t$ and $v_{s,t}=\EE{\one{x_s^{\pi_t} = x} \middle| \cF_T}$ is the distribution of $x_s^{\pi_t}$ for $s\le t$.\footnote{Notice that $\cF_T$ contains only policies, which are independent of the state variables.} Let $E_t$ be the event of a policy switch at time $t$. From inequality
\[
\norm{\pi_{t-k} - \pi_t}_{\infty,1} \le \norm{\pi_{t-k} - \pi_{t-k+1}}_{\infty,1} + \dots + \norm{\pi_{t-1} - \pi_t}_{\infty,1} \le 2 \sum_{s=t-k+1}^t \one{E_s} \,,
\]
and Lemma~\ref{lem:policy-switch}, we get that
\beq
\label{eq:eqn1}
\EE{\norm{\pi_{t-k} - \pi_t}_{\infty,1}} \le 2 \sqrt{\frac{\log |\Pi|}{T}} k \;.
\eeq
Let $P_t^{\pi} = P(\pi, m_t)$. We have that
\begin{align}
\label{eq:eq1}
\notag
\EE{\norm{u_{t} - v_{t,t}}_1} &= \EE{\norm{u_{t-1} P_{t-1}^{\pi_{t-1}} - v_{t-1,t} P_{t-1}^{\pi_{t}} }_1} \\
\notag
&= \EE{\norm{u_{t-1} P_{t-1}^{\pi_{t-1}} - u_{t-1} P_{t-1}^{\pi_{t}} + u_{t-1} P_{t-1}^{\pi_{t}} - v_{t-1,t} P_{t-1}^{\pi_{t}} }_1} \\
\notag
&\le \EE{\norm{u_{t-1} P_{t-1}^{\pi_{t-1}} - u_{t-1} P_{t-1}^{\pi_{t}} }_1 + \norm{u_{t-1} P_{t-1}^{\pi_{t}} - v_{t-1,t} P_{t-1}^{\pi_{t}} }_1} \\
\notag
&\le \EE{\norm{\pi_{t-1} - \pi_t}_{\infty,1} + e^{-1/\tau} \norm{u_{t-1} - v_{t-1,t}}_1} \\
\notag
&\le \E\Big[\norm{\pi_{t-1} - \pi_t}_{\infty,1} + e^{-1/\tau} (\norm{u_{t-2} P_{t-2}^{\pi_{t-2}}  - u_{t-2}  P_{t-2}^{\pi_{t}}  }_1 \\
\notag
&\qquad\qquad+ \norm{u_{t-2} P_{t-2}^{\pi_{t}} - v_{t-2,t} P_{t-2}^{\pi_{t}} }_1)\Big] \\
\notag
&\le \EE{\norm{\pi_{t-1} - \pi_t}_{\infty,1} + e^{-1/\tau} \norm{\pi_{t-2} - \pi_t}_{\infty,1} + e^{-2/\tau} \norm{u_{t-2} - v_{t-2,t}}_1} \\
\notag
&\le \dots \\
\notag
&\le \sum_{k=0}^t e^{-k/\tau} \EE{\norm{\pi_{t-k} - \pi_t}_{\infty,1}} + e^{-t/\tau} \norm{u_{0} - v_{0,t}}_1  \\
\notag
&\le \sum_{k=0}^t 2 e^{-k/\tau} \sqrt{\frac{\log |\Pi|}{T}} k + 0 \qquad\text{By \eqref{eq:eqn1}}\\
&\le 2 \sqrt{\frac{\log |\Pi|}{T}} \tau^2 \, ,
\end{align}
where we have used the fact that $\norm{u_{0} - v_{0,t}}_1 = 0$, because the initial distributions are identical.
By \eqref{eq:eq1} and \eqref{eq:BT}, we get that
\[
\EE{B_T(A)} \le 2 \tau^2 \sum_{t=1}^T   \sqrt{\frac{\log |\Pi|}{T}}  = 2 \tau^2 \sqrt{\log |\Pi|T}  \;.
\]
\end{proof}
What makes the analysis possible is the fact that all policies mix no matter what transition model is played by the adversary.
\begin{proof}[Proof of Theorem~\ref{thm:main}]
The result is obvious by Lemmas~\ref{lem:second-part} and \ref{lem:first-part}.
\end{proof}

The next corollary extends the result of Theorem~\ref{thm:main} to continuous policy spaces. 
\begin{cor}
Let $\Pi$ be an arbitrary policy space, $\cN(\epsilon)$ be the $\epsilon$-covering number of space $(\Pi, \norm{.}_{\infty,1})$, and $\cC(\epsilon)$ be an $\epsilon$-cover. Assume that we run the \textsc{SD-MDP} algorithm on $\cC(\epsilon)$. Then, under the same assumptions as in Theorem~\ref{thm:main}, for any policy $\pi\in \Pi$, 
\[
\EE{R_T(\textsc{SD-MDP},\pi)} \le (4 + 2 \tau^2) \sqrt{T \log\cN(\epsilon)} + \log\cN(\epsilon) + \tau T \epsilon  \; .
\]
\end{cor}
\begin{proof}
Let $L_T(\pi) = \EE{\sum_{t=1}^T \ell_t(x_t^{\pi}, \pi)}$ be the value of policy $\pi$. Let $u_{\pi,t}(x) = \Prob{x_t^\pi = x}$. First, we prove that the value function is Lipschitz with Lipschitz constant $\tau T$. The argument is similar to the argument in the proof of Lemma~\ref{lem:first-part}. For any $\pi_1$ and $\pi_2$,
\begin{align*}
\abs{L_T(\pi_1) - L_T(\pi_2)} &= \abs{\EE{\sum_{t=1}^T \ell_t(x_t^{\pi_1}, \pi_1) - \sum_{t=1}^T \ell_t(x_t^{\pi_2}, \pi_2)}} \\
&\le 2\abs{\sum_{t=1}^T \norm{u_{\pi_1,t} - u_{\pi_2,t}}_1 \norm{\ell_t}_{\infty}  } \\
&\le 2\abs{\sum_{t=1}^T \norm{u_{\pi_1,t} - u_{\pi_2,t}}_1} \; .
\end{align*}
With an argument similar to the one in the proof of Lemma~\ref{lem:first-part}, we can show that
\[
\norm{u_{\pi_1,t} - u_{\pi_2,t}}_1 \le \tau \norm{\pi_1 - \pi_2}_{\infty,1}\;.
\]
Thus,
\[
\abs{L_T(\pi_1) - L_T(\pi_2)} \le \tau T \norm{\pi_1 - \pi_2}_{\infty,1} \; .
\]
Given this and the fact that for any policy $\pi\in\Pi$, there is a policy $\pi'\in\cC(\epsilon)$ such that $\norm{\pi - \pi'}_{\infty,1} \le \epsilon$, we get that
\[
\EE{R_T(\textsc{SD-MDP},\pi)} \le  (4 + 2 \tau^2) \sqrt{T \log\cN(\epsilon)} + \log\cN(\epsilon) + \tau T \epsilon  \; .
\]
\end{proof}
In particular if $\Pi$ is the space of all policies, $\cN(\epsilon)\le(|\cA|/\epsilon)^{|\cA||\cX|}$, so regret is no more than
\[
\EE{R_T(\textsc{SD-MDP},\pi)} \le (4 + 2 \tau^2) \sqrt{T |\cA| |\cX| \log\frac{|\cA|}{\epsilon}} + |\cA| |\cX| \log\frac{|\cA|}{\epsilon} + \tau T \epsilon  \; .
\]
By the choice of $\epsilon = \frac{1}{T}$, we get that $\EE{R_T(\textsc{SD-MDP},\pi)} = O(\tau^2\sqrt{T \abs{\cA}\abs{\cX} \log (|\cA|T)})$.


\bibliography{all_bib}

\end{document}

%% file: Commands.tex
\usepackage{pdfsync}
\usepackage{graphicx} 

\usepackage{algorithm}
\usepackage{algorithmic}

\usepackage{amssymb}

\usepackage{amsmath} 

\usepackage{color}
\usepackage{graphicx}
\usepackage{epstopdf}
\usepackage{algorithm,algorithmic}
\usepackage{verbatim}

\usepackage{nicefrac}

\newif\ifcomm
\commtrue

\newif\iflong
\longtrue

\newtheorem{thm}{Theorem}
\newtheorem{cor}[thm]{Corollary}
\newtheorem{lem}[thm]{Lemma}

\newcounter{assumption}
\renewcommand{\theassumption}{A\arabic{assumption}}

\newenvironment{ass}[1][]{\begin{trivlist}\item[] \refstepcounter{assumption}%
 {\bf Assumption\ \theassumption\ #1} }{
 \ifvmode\smallskip\fi\end{trivlist}}

\newcommand{\norm}[1]{\left\Vert#1\right\Vert}

\newcommand{\abs}[1]{\left\vert#1\right\vert}

\newcommand{\Real}{\mathbb R}                        

\newcommand{\Prob}[1]{{\mathbb P}\left(#1\right)}    
\newcommand{\EE}[1]{{\mathbb E}\left[#1\right]}      
\newcommand{\E}{{\mathbb E}}                         
\newcommand{\one}[1]{{\mathbb I}_{\{#1\}}}           

\newcommand{\ra}{\rightarrow}

\newcommand{\ip}[2]{\langle #1,#2 \rangle}

\newcommand{\beq}{\begin{equation}}
\newcommand{\eeq}{\end{equation}}
\newcommand{\beqa}{\begin{eqnarray}}
\newcommand{\eeqa}{\end{eqnarray}}
\newcommand{\beqan}{\begin{eqnarray*}}
\newcommand{\eeqan}{\end{eqnarray*}}
\newcommand{\ben}{\begin{eqnarray*}}
\newcommand{\een}{\end{eqnarray*}}
\newcommand{\bea}{\begin{align*}}
\newcommand{\eea}{\end{align*}}

\ifcomm
   \newcommand\comm[1]{\textcolor{blue}{ #1}}
   \newcommand{\mtodo}[2]{\todo{{\bf #1}: #2}} 
   \def\here#1{{\bf $\langle\langle$#1$\rangle\rangle$}}

\else
   \newcommand\comm[1]{}
   \newcommand{\mtodo}[2]{}
   \def\here#1{}
\fi

\newcommand{\cX}{{\cal X}}

\newcommand{\cC}{{\cal C}}
\newcommand{\cA}{{\cal A}}

\newcommand{\cF}{{\cal F}}

\renewcommand{\phi}{\varphi}

\newcommand{\cN}{{\cal N}}

%% file: onlineRL-arxiv.bbl
\begin{thebibliography}{9}
\providecommand{\natexlab}[1]{#1}
\providecommand{\url}[1]{\texttt{#1}}
\expandafter\ifx\csname urlstyle\endcsname\relax
  \providecommand{\doi}[1]{doi: #1}\else
  \providecommand{\doi}{doi: \begingroup \urlstyle{rm}\Url}\fi

\bibitem[Cesa-Bianchi and Lugosi(2006)]{Cesa-Bianchi-Lugosi-2006}
Nicol\`{o} Cesa-Bianchi and G\'{a}bor Lugosi.
\newblock \emph{Prediction, Learning, and Games}.
\newblock Cambridge University Press, New York, NY, USA, 2006.

\bibitem[Even-Dar et~al.(2004)Even-Dar, Kakade, and
  Mansour]{Even-Dar-Kakade-Mansour-2004}
Eyal Even-Dar, Sham~M. Kakade, and Yishay Mansour.
\newblock Experts in a {M}arkov decision process.
\newblock In \emph{NIPS}, 2004.

\bibitem[Even-Dar et~al.(2009)Even-Dar, Kakade, and
  Mansour]{Even-Dar-Kakade-Mansour-2009}
Eyal Even-Dar, Sham~M. Kakade, and Yishay Mansour.
\newblock Online {M}arkov decision processes.
\newblock \emph{Mathematics of Operations Research}, 34\penalty0 (3):\penalty0
  726--736, 2009.

\bibitem[Geulen et~al.(2010)Geulen, V\"{o}cking, and
  Winkler]{Geulen-Vocking-Winkler-2010}
Sascha Geulen, Berthold V\"{o}cking, and Melanie Winkler.
\newblock Regret minimization for online buffering problems using the weighted
  majority algorithm.
\newblock In \emph{COLT}, 2010.

\bibitem[Littlestone and Warmuth(1994)]{Littlestone-Warmuth-1994}
Nick Littlestone and Manfred~K. Warmuth.
\newblock The weighted majority algorithm.
\newblock \emph{Inf. Comput.}, 108\penalty0 (2):\penalty0 212--261, 1994.

\bibitem[Neu et~al.(2010)Neu, Gy\"{o}rgy, and
  Csaba~Szepesv\'{a}ri]{Neu-Gyorgy-Szepesvari-Antos-2010}
Gergely Neu, Andr\'{a}s Gy\"{o}rgy, and Andr\'{a}s~Antos Csaba~Szepesv\'{a}ri.
\newblock Online {M}arkov decision processes under bandit feedback.
\newblock In \emph{NIPS}, 2010.

\bibitem[Vovk(1990)]{Vovk-1990}
Vladimir Vovk.
\newblock Aggregating strategies.
\newblock In \emph{COLT}, pages 372--383, 1990.

\bibitem[Yu and Mannor(2009{\natexlab{a}})]{Yu-Mannor-2009}
Jia~Yuan Yu and Shie Mannor.
\newblock Arbitrarily modulated {M}arkov decision processes.
\newblock In \emph{IEEE Conference on Decision and Control},
  2009{\natexlab{a}}.

\bibitem[Yu and Mannor(2009{\natexlab{b}})]{Yu-Mannor-2009-b}
Jia~Yuan Yu and Shie Mannor.
\newblock Online learning in {M}arkov decision processes with arbitrarily
  changing rewards and transitions.
\newblock In \emph{GameNets}, 2009{\natexlab{b}}.

\end{thebibliography}
